\documentclass{tlp}

\usepackage[space=true]{accsupp}
\usepackage{amsmath}
\usepackage{amssymb}
\usepackage{ dsfont }
\usepackage{epigraph}
\usepackage{graphicx}
\usepackage{multirow}
\usepackage{paralist} 
\usepackage{ stmaryrd }
\usepackage{url}
\usepackage{wrapfig}
\usepackage{xspace}
\usepackage{natbib}

\newcommand\m[1]{\ensuremath{#1}\xspace}
\newcommand{\T}    {\m{T}}
\newcommand{\Tenv} {\m{T_{\mathit{env }}}}

\newcommand{\Tconf}{\m{T_{\mathit{sol}}}}

\newcommand{\Tall} {\m{T_{\mathit{all }}}}

\newcommand{\Venv}{\m{\Sigma_{\mathit{env} }}}
\newcommand{\Vconf}{\m{\Sigma_{\mathit{sol}}}}
\newcommand{\Vall}{\m{ \Sigma_{\mathit{all}}}}

\newcommand{\Tr}{\m{\mathbf{t}}}
\newcommand{\Fa}{\m{\mathbf{f}}}
\newcommand{\Un}{\m{\mathbf{u}}}

\newcommand\fodot{\m{\mathrm{FO}(\cdot)}}
\newcommand{\voc}{\Sigma}

\newcommand{\I}{\m{S}}
\newcommand{\J}{\m{S'}}
\newcommand{\M}{\m{M}}
\newcommand{\dom}[1]{\m{dom(#1)}}
\newcommand{\intsym}[1]{\Sigma_{#1}}

\newcommand{\compl}[1]{\overline{#1}}

\newcommand{\eval}[2]{\llbracket #1 \rrbracket^{#2}}

\newcommand{\Senv}{\m{S_\mathit{env}}}
\newcommand{\Sobs}{\m{S_\mathit{obs}}}

\newcommand{\Sobsi}{\m{S_{\mathit{obs} }^{\mathit{i}}}}
\newcommand{\Sobsj}{\m{S_{\mathit{obs} }^{\mathit{i+1}}}}
\newcommand{\Se}{\m{S_o}}
\newcommand{\Sdec}{\m{S_{\mathit{dec}} }}

\newcommand{\Sdeci}{\m{S_{\mathit{dec}}^{\mathit{i}  }}}
\newcommand{\Sdecj}{\m{S_{\mathit{dec}}^{\mathit{i+1}}}}
\newcommand{\Sc}{\m{S_d}}
\newcommand{\Sconf} {\m{S_{\mathit{sol}}}}
\newcommand{\Sall}{\m{S_{\mathit{all} }}}
\newcommand{\MinDef}{\m{MinDef}}
\newcommand{\MinCont}{\m{MinCont}}

\newcommand{\RegistrationType}{\m{\mathit{RegistrationType} }}
\newcommand{\SocialHabitat}{\m{\mathit{SocialHousing} }}
\newcommand{\LowRent}{\m{\mathit{LowRent} }}
\newcommand{\TaxRate}{\m{\mathit{TaxRate} }}
\newcommand{\LicensedSeller}{\m{\mathit{LicensedSeller} }}
\newcommand{\Social}{\m{\mathit{Social} }}
\newcommand{\Modest}{\m{\mathit{Modest} }}
\newcommand{\Other}{\m{\mathit{Other} }}

\usepackage[dvipsnames,svgnames]{xcolor}
\newcommand{\pierre}[1]{{\color{red}\textsc{PC:} #1}}

\newcommand{\bart}[1]{{\color{OliveGreen}\textsc{BB:} #1}}

\newcommand{\ignore}[1]{}

\renewcommand\implies{\Rightarrow}

\newcommand{\leqp}{\leq_p}

\newcommand{\ltp}{<_p}

\newcommand{\join}[2]{\m{#1+#2}}

\newcommand{\structure}{structure\xspace}

\newcommand{\System}{{\tt System}\xspace}
\newcommand{\User}{{\tt User}\xspace}
\newcommand{\cautious}{{definite}\xspace}
\newcommand{\brave}{{contingent}\xspace}

\newtheorem{definition}{Definition}
\newtheorem{example}{Example}
\newtheorem{proposition}{Proposition}

\begin{document} 

\lefttitle{Interactive Model Expansion in an Observable Environment - Carbonnelle et al.}

\jnlPage{1}{8}
\jnlDoiYr{2021}
\doival{10.1017/xxxxx}

\title[Interactive Model Expansion in an Observable Environment]{Interactive Model Expansion\\ in an Observable Environment\thanks{This research received funding from the Flemish Government under the ``Onderzoeksprogramma Artifici\"ele Intelligentie (AI) Vlaanderen'' programme.}}

\begin{authgrp}
    \author{\sn{Carbonnelle} \gn{Pierre}}\affiliation{KU Leuven}
    \author{\sn{Vennekens} \gn{Joost} } \affiliation{KU Leuven}
    \author{\sn{Bogaerts} \gn{Bart}}    \affiliation{Vrije Universiteit Brussel}
    \author{\sn{Denecker} \gn{Marc} }   \affiliation{KU Leuven}
\end{authgrp}


\maketitle

\begin{abstract}
    Many practical problems can be understood as the search for a state of affairs that extends a fixed partial state of affairs, the \emph{environment}, while satisfying certain conditions that are formally specified.
    Such problems are found in, e.g., engineering, law or economics.

    We study this class of problems in a context where some of the relevant information about the environment is not known by the user at the start of the search.
    During the search, the user may consider tentative solutions that make implicit hypotheses about these unknowns.
    To ensure that the solution is appropriate, these hypotheses must be verified by observing the environment.
    Furthermore, we assume that, in addition to knowledge of what constitutes a solution, knowledge of general laws of the environment is also present.
    We formally define partial solutions with enough verified facts to guarantee the existence of complete and appropriate solutions.

    Additionally, we propose an interactive system to assist the user in their search by determining 1)  which hypotheses implicit in a tentative solution must be verified in the environment, and 2) which observations can bring useful information for the search.
    We present an efficient method to over-approximate the set of relevant information, and evaluate our implementation.

\end{abstract}


\section{Introduction}
    \setlength{\epigraphwidth}{2in}
    \renewcommand*{\textflush}{flushright}
    \setlength{\epigraphrule}{0pt}
    \epigraph{Some things are in our control \\and others not.}{\textit{Epictetus}}

Many practical problems can be understood as the search for a state of affairs that extends a fixed partial state of affairs, the \emph{environment}, while satisfying certain conditions that are formally specified.
Such problems are found in, e.g., engineering, law or economics.
Consider, for instance, the buyer of a house who has to declare the purchase for tax purposes.
The environment consists of the properties of the house, of the seller and of the sales contract: they are not under the control of the purchaser.
This environment has to be extended with a registration containing decisions on, e.g., which tax rebate to claim.
Such rebates are subject to conditions, and have consequences, stated by law.

The Knowledge Representation (KR) methodology to solve such problems is to design a vocabulary to represent the relevant objects and concepts in the states of affairs, and to express the conditions as a theory in a declarative KR language.
The theory is the set of laws to be satisfied.
When the environment is fully known, it can be expressed as a structure, and the search for a solution can be performed automatically by Model Expansion inference on the theory and the structure~\citep{mitchell2006model}.

However, the information about the environment is rarely fully known at the start of the search.
In our example, the buyer of the house does not know all the attributes of the house when they start their registration.
In such cases, two approaches have been implicitly used:
(i) all \emph{relevant} information about the environment is assumed to be known at the start of the search so that a solution can be determined without making any hypothesis about the environment~\citep{Felfernig2014b}, or,
(ii) when this assumption is dropped, the model expansion mechanism will choose values for the unknown environmental symbols, and the user has the implicit responsibility to verify that these values are indeed correct in the environment.
We are not aware of any method to assist the user in the verification task that this second approach requires.
Thus, there is a danger that the verification is incomplete, and that the user accepts a solution making incorrect assumptions in their environment.

The primary goal of this work is to study the problem of model expansion in an unknown, but observable, environment in a precise and formal way.
To achieve this goal, it is crucial that we make an \emph{explicit distinction} between environmental and decision symbols in the system.
In this setting, an important difference arises between an \emph{unknown} environmental symbol, and an \emph{undecided} decision symbol.
In both cases, the value of the symbol is not available to the system, but in the first case, it exists but is unknown, while in the second case, it does not exist yet.
Values assigned to symbols by the system must be handled accordingly.
When the system assigns a value to an environmental symbol, the user must go observe if the assigned value is correct.

Notice that the value of some environmental symbols may not be known at the start of the search.
For example, the buyer may not know whether a rebate has already been claimed for a renovation.
If this information becomes relevant, the buyer will have to ``observe it'' in the environment, e.g., by asking the seller.
Distinguishing environmental from decision symbols is unique to our system.
When the system is aware of this distinction, it can assist the user by highlighting which observations still need to be made to be sure that the proposed solution conforms to reality, as well as inform them that the final configuration is indeed appropriate (i.e., that all the hypotheses have been verified by observing the environment).

Next to making a distinction between two types of symbols, we also make a distinction between the laws of possible environments and the laws of acceptable solutions.
The laws of possible environments are guaranteed to hold independently of the decisions of the user; by contrast, the satisfaction of the laws of acceptable solutions depends on the decision of the user.
Importantly, unlike consequences of the laws of possible environments, environmental consequences derived from the laws of acceptable solutions must be observed.
To see this, assume we extend our previous examples with assertions stating which county belongs to which states, and what decisions a buyer must make to obtain a tax rebate in each state.
The first assertions are part of the laws of possible environments, while the second assertions are part of the laws of acceptable solutions.
If the property is known to be in a county, it is safe to conclude from the first set of assertions that the property is also in the state of that county: it is not necessary to observe the state in which it is.
By contrast, suppose that the buyer decides to claim a tax rebate available in only one state.
By the second set of assertions, the system concludes that the property must be in that state, but this is not a safe conclusion.
Indeed, it is clear that this conclusion must be verified by observing the environment, and if the property is observed to be in another state, the user must be informed that he cannot claim that tax rebate.

To further assist the user in their task, the system can also determine which observations can bring useful information for the search.
Suppose that a lower tax rate can be obtained only if the property is energy efficient and will be used for housing.
If the buyer does not plan to use it for housing, then the energy efficiency does not need to be known.
To capture this kind of scenario, we introduce a notion of \emph{relevance}: informally, a symbol is relevant when it is an essential part of a solution.
Relevance  is instrumental in reducing the number of observations and decisions the user has to make.
This is especially valuable when such observations are costly.
Past efforts to determine relevant symbols in the context of interactive configuration have used so-called \emph{activation rules}, i.e., rules that determine when symbols become relevant.
These activation rules have to be added manually by the knowledge engineer.
Our approach eliminates the need for such rules, and automatically derives the relevance of symbols based on their occurrence in minimal solutions instead.
Furthermore, we propose an approximate, but efficient, method to derive relevance automatically, using standard solvers.

The rest of this paper is organized as follows.
In Section~\ref{sec:prelims}, we introduce the logic concepts we use in this paper.
Next, we formalize model expansion in a two-theory setting: one describing valid solutions and another describing possible environments, with a clear distinction between environmental and decision symbols (Section \ref{total solutions}).
In Section~\ref{process} and \ref{Correctness}, we formalize the concept and operation of an assistant for interactive model expansion in that setting.
Then, we define partial solutions with enough observations and decisions to guarantee that complete and appropriate solutions exist (Section~\ref{Stop}).
We define relevant symbols (Sections~\ref{Relevance}), and present an efficient method to over-approximate the set of relevant symbols.
Finally, we evaluate our implementation (Section~\ref{Implementation}) and compare it to related work (Section~\ref{related})

\section{Preliminaries}
\label{sec:prelims}

We assume familiarity with syntax and semantics of propositional logic and first-order logic (FO).
To keep the presentation light, we develop our theory for propositional logic, but, in fact, the concepts we define apply to all logics with a model semantics.
Our examples use a fragment of first-order logic.

\ignore{
    \bart{Should we mention "in the context of a fixed and finite domain"?}
    \pierre{Z3 supports infinite domains.  We would say somehting about fixed though}
    \bart{If you want to guarantee termination of the system, finiteness might become important. Or you need to be very careful, for instance not allowing relations/functions over the infinite domain (infinitely many choices to be made?)...}
    \pierre{I understand, but this does not belong in this section.  Also, we do not make claim about termination.}
}

A vocabulary $\voc$ is a set of symbols.
A theory is a set of formulae constructed from these symbols using the logic connectives $\lnot, \land, \lor, \implies, \Leftrightarrow$ (negation, conjunction, disjunction, material implication, equivalence).
A partial (resp. total) structure of $\voc$ is a partial (resp. total) function from $\voc$ to $\{\Tr,\Fa\}$, giving a Boolean value (a.k.a. interpretation) to each symbol.
\ignore{
    More generally, a structure is a mapping of symbols to their interpretation in a domain of discourse.\footnote{In first-order logic, a partial structure is a function from symbols $\sigma\in\voc$ to partial (total) values $\sigma^\I$, where a partial (total) $n$-ary relation is a partial (total) function from $\dom{\I}^n$ to $\{\Tr,\Fa\}$ (or a total function to $\{\Tr,\Fa,\Un\}$, respectively to $\{\Tr,\Fa\}$), and a partial (total) $n$-ary function is a partial (total) function from $\dom{\I}^n$ to $\dom{\I}$. }
}
The value of a symbol $\sigma\in\voc$ in a structure $S$ is denoted by $\sigma^\I$ if it exists.
In this case, we say that $\sigma$ is interpreted by $\I$, and that the \emph{fact} ``$\sigma$ has the value denoted by $\sigma^\I$'' is true in $\I$. The set of interpreted symbols of $\I$ is denoted by $\intsym{\I}$.
If $\I$ is a total structure of $\voc$, we can determine the truth-value of all well-formed expressions and theories $e$ over $\voc$; the resulting truth-value is denoted by $\eval{e}{\I}$.
A total structure \I that makes every formula in \T true is a \emph{model} of \T; then, \T is \emph{satisfied} by \I.

We call a partial structure $\I$ \emph{less precise than}  $\J$ (and write $\I \leqp \J$) if $\Sigma_\I \subseteq \Sigma_{\J}$ and all facts true in $\I$ are also true in $\J$, i.e., if $\J$ assigns the same values as $\I$ to the interpreted symbols of $\I$ (it may interpret more symbols).
In this case, we call \J an \emph{expansion of \I}.
The set of total expansions of $\I$ is denoted $\compl{\I}$.
A partial structure \I is \emph{consistent with \T} if it has a \emph{model expansion} for \T, i.e., a total expansion that satisfies \T. 
It \emph{entails \T} if all its total expansions satisfy \T.
We say that $\I$ and $\J$ are \emph{disjoint} if  $\intsym{\I}\cap\intsym{\J}=\emptyset$.
For two disjoint structures $\I$ and $\J$, we define $\join{\I}{\J}$ to be the structure that maps symbols $\sigma$ interpreted by $\I$  to $\sigma^\I$ and symbols interpreted by \J  to $\sigma^{\J}$.

\ignore{

    \textit{Model expansion} is the inference that takes as input a partial structure $S$ and a theory $T$, and has output  a total structure $\J$ expanding  $S$, such that $\eval{T}{S}$ is $\Tr$.
    If there is such a structure, we write $\exists \J \in \compl{S}: \eval{T}{\J} =\Tr$, where $\compl{S}$ is the set of expansions of $S$.
    If every expansion of $S$ satisfies $T$, we write $\forall \J \in \compl{S}:  \eval{T}{\J} =\Tr.$
}

\section{Split model expansion}
\label{total solutions}

We first consider the generic problem of (non-interactive) model expansion in a two-theory setting, using two vocabularies, when the environment is fully known.
The first vocabulary, \Venv, contains the \emph{\underline{env}ironmental symbols}.
The particular environment faced by the user is described by a total \structure \Senv for \Venv. The second vocabulary, \Vconf, is disjoint from the first and contains the decision symbols.
A \underline{sol}ution is fully described by a total \structure \Sconf for \Vconf.
The union of the 2 vocabularies is denoted by \Vall.

We assume to have a theory \Tenv over vocabulary \Venv that contains the laws satisfied by any environment that could be faced by the user.
Any structure \Senv that describes a particular environment is a model of \Tenv.
We also assume that the constraints on the solution to the problem faced by the user are formalized by a logic theory \Tconf in the vocabulary \Vall.
This theory fully captures all knowledge about when a solution is valid in a given environment.

\begin{definition}
\label{def:enhanced}
    The \emph{split model expansion problem} $\mathcal{SMX}$ is a problem that takes as input a tuple $(\Venv, \Vconf, \Senv, \Tenv, \Tconf)$ where \Senv is a total structure over \Venv, \Tenv is a theory over \Venv satisfied by \Senv, and \Tconf is a theory over $\Venv \cup \Vconf$.\\
    A \emph{total solution} for such a problem is a total structure \Sconf over \Vconf such that
    \[\eval{\Tconf}{\join{\Senv}{\Sconf}}=\Tr.\]
\end{definition}

Note that \Tenv is not used in the definition of the total solution because we assume that the environment satisfies \Tenv already.
Having $\Tenv$ as an additional input is useful in the process of finding a solution when the environment  \Senv is not fully known, but observable, as discussed in the next sections.

\ignore{
    By contrast, prior work (e.g.,~\cite{van2017kb}) does not explicitly split the vocabulary in two sub-vocabularies.
    It considers problems whose input is a tuple $(\Vall, \Sall, \Tall)$ where \Sall is a partial structure over \Vall.
    Such problems can be seen as a special case of our split  model expansion problem $\mathcal{SMX}$ where \Venv and \Tenv are empty, and \Vconf, \Senv and \Tconf are set to \Vall, \Sall and \Tall respectively.
}

\begin{example}
\label{ex:wet}
    We consider a simplified tax legislation for the sale of real estate, defined as follows.
    The tax amount is 10 $\%$ of the sales price for standard registration with the authorities, but a lower tax rate can be obtained when the house is bought to be leased to tenants for a low rent, or when the seller is licensed to sell social housing.

    The environmental symbols we use to describe the seller and the house are propositions; the decision symbols are nullary function symbols: the type of registration (whose value is either $\Social, \Modest$ or $\Other$) and the tax rate (an integer).
    The states of affairs are described by structures over these vocabularies:
    \begin{align*}
        \Venv &= \{\SocialHabitat, \LicensedSeller , \LowRent \}. \\
        \Vconf &= \{\RegistrationType, \TaxRate\}.
    \end{align*}
    The theory \Tenv  about the possible environments states that a social housing is necessarily leased at a low rent.
    \[\SocialHabitat \Rightarrow \LowRent.\]
    The legislation is described by \Tconf, stating the prerequisites for the types of registration, and the corresponding tax rate (where equality is interpreted as in first order logic):
    \begin{align*}
       &\RegistrationType {=} \Modest  \Rightarrow \LowRent.\\
       &\RegistrationType {=} \Social  \Rightarrow \LicensedSeller \land \SocialHabitat.\\
       &\TaxRate {=} 1~  \Leftrightarrow \RegistrationType {=} \Social. \\
       &\TaxRate {=} 7~  \Leftrightarrow \RegistrationType {=} \Modest.~ \\
       &\TaxRate {=} 10  \Leftrightarrow \RegistrationType {=} \Other.~~
    \end{align*}
    A possible sale is described by \Senv {=} $\{\SocialHabitat{=}\Tr, LicencedSeller{=}\Tr, \LowRent{=}\Tr\}$.
    The split model expansion problem has 3 total solutions:
    \begin{align*}
        &\{\TaxRate {=} 1, \RegistrationType {=} \Social\}, \\
        &\{\TaxRate {=} 7, \RegistrationType {=} \Modest.\}, \text{ and}\\
        &\{\TaxRate {=} 10, \RegistrationType {=} \Other.\}.
    \end{align*}
    An interactive assistant for this theory is available online.\footnote{\url{https://tinyurl.com/simpleRegistration}}
\end{example}

\section{Interactive model expansion}
\label{process}

Often, a problem has many acceptable solutions in a given environment.
Instead of choosing a particular solution by automated model expansion~\citep{mitchell2006model}, it is preferable that the automated system allow the user to explore the set of acceptable solutions.
For example, a property tax legislation may allow the buyer to register the sales in different ways, as in Example 1, and the system should allow them to see the various prerequisites and consequences of each valid options, e.g., on the tax rate.

\ignore{
    In that case, the model expansion problem becomes a distributed constraint satisfaction problem,~\citep{yokoo2012distributed}, i.e., a problem where two agents, the automated system and its user, want to conceive a course of action that satisfies both of them.
    Each agent has its own knowledge, constraints and goals.
    We formalize this approach as follows.
}

We call the two interacting agents \System and \User.
\System offers \User the freedom to choose values for environmental and decision symbols while verifying the validity of the solution that \User constructs with respect to the 2 theories, \Tenv and \Tconf.
Since a system is often designed for many users, and each user wants a solution for their own environment,
we assume it is the task of \User to communicate information about their environment to \System.
\User may have additional knowledge that is unknown to \System.
Indeed, it is often the case that some tacit knowledge is not formalized in the theories, but well known by the user.
Obviously, \System cannot use this tacit knowledge; instead, \User must use it in the appropriate way.

At each step of an interactive process, \User has constructed two partial structures, \Sobs and \Sdec, assigning values to environmental and decision symbols respectively.
Those values correspond to \underline{obs}ervations and to (tentative) \underline{dec}isions respectively.

\begin{definition}
    A \emph{state of the interactive process} is a pair of partial structures $(\Sobs, \Sdec)$ such that \Sobs (resp. \Sdec) is a partial structure over \Venv (resp. \Vconf).
\end{definition}

By contrast, in prior work where there is no distinction between environmental and decision symbols, the state is described by only one partial structure.

Although \System only has the information that \User has provided to it, it also knows \Tenv and \Tconf; hence, it can verify that the state is consistent with \Tenv and \Tconf:

\begin{definition}
    \label{consistent}
    A state $(\Sobs, \Sdec)$ is \emph{consistent} (with \Tenv and \Tconf) iff there is an expansion of the state that satisfies both theories:
    \[\exists S \in \compl{\join{\Sobs}{\Sdec}} : \eval{\Tenv \cup \Tconf}{S} = \Tr\]
\end{definition}

During the interactive process, decisions can be added and retracted. This leads to the following definition:

\begin{definition}
    An \emph{interactive process} for a problem $\mathcal{SMX}$ with input $(\Venv, \Vconf, \Senv, \Tenv, \Tconf)$ is a sequence $(\Sobsi,\Sdeci)_{0<i\leq n}$ of $n$ consistent states, such that, at each step, either the set of observations or the set of decisions is made either more precise or less precise, i.e., for any $0<i<n$:
    \begin{compactitem}
        \item $(\Sobsi \ltp \Sobsj \lor \Sobsj \ltp \Sobsi) \land \Sdeci = \Sdecj$
        \item or, $\Sobsi = \Sobsj \land (\Sdeci \ltp \Sdecj \lor \Sdecj \ltp \Sdeci)$.
    \end{compactitem}
\end{definition}


In the interactive process, \User and \System work together towards a solution as follows: 
\begin{compactitem}
    \item The interaction starts with a pair of empty structures, $(\Sobs, \Sdec)$.
            If this state is inconsistent with \Tenv and \Tconf, \System tells \User that there is no solution to the split model expansion problem.

    \item
        \User can make \Sobs or \Sdec more precise as long as the new state is consistent.
        \System prevents \User from changing their decisions \Sdec to an inconsistent state because it would not bring them closer to a solution.
        If \User needs to make \Sobs more precise following an observation, and if this change makes the state inconsistent, \System asks \User to retract a decision first so that the new observation can then be added without making the state inconsistent.
        To facilitate this task, \System shows the list of decisions that, if retracted, would allow the entry.
        This list is the list of interpretations in the $\leq_p$-minimal structure $S$ less precise than $\Sdec$ that, together with \Sobs, is inconsistent with $\Tenv \bigcup \Tconf \bigcup O$, where $O$ is the observation to be added.
        It can be obtained using unsat\_core methods~\citep{DBLP:conf/sat/CimattiGS07}.
        By repeatedly making the state more precise while keeping it consistent, \User eventually obtains a total solution.

    \item
        \User can make \Sobs or \Sdec less precise, i.e., retract past observations and decisions.
        The resulting state is necessarily consistent.
        This allows \User to explore the complete space of solutions in their environment.
\end{compactitem}

\ignore{By contrast, Figure~\ref{fig:state}(b) shows the interactive process in prior work.}

\begin{example}
    Returning to Example 1, assume that \User first decides that $\RegistrationType{=} \Social$.
    The corresponding state $(\{\}, \{\RegistrationType {=} \Social\})$ is consistent.
    They proceed by observing $\lnot \SocialHabitat$.
    This cannot be added to the state because \SocialHabitat is a prerequisite for a \Social registration, per \Tconf.
    \User thus needs to retract their only decision ($\RegistrationType{=} \Social$) before entering their observation.

    Suppose instead that \User selects $\RegistrationType {=} \Social$ followed by $\LowRent {=}\Fa$.
    The resulting state is consistent with \Tconf in the absence of \Tenv.
    However, the state is inconsistent with \Tenv because a \Social registration can only be done for a social housing, which is necessarily leased at a low rent.
    \User must backtrack on their decision $\RegistrationType {=}\Social$ before entering the observation that the housing is not leased at a low rent.
\end{example}

Note that the presence of \Tenv reduces the number of consistent states, and thus shortens the search for a solution.

\section{Propagation}
\label{Correctness}

\System further assists \User by informing them of the logical consequences of what is known in the current state.
Such knowledge is of great value to \User, and is used in most interactive assistant tools  \citep{van2017kb,falkner2019constraint}.
It reduces the number of decisions that \User has to make, and shortens the interaction in a way analogous to how unit propagation accelerates the naive backtracking algorithm in propositional satisfiability checking.
The consequences of what is known are determined by a computation called \emph{propagation} (or \emph{backbone} computation).

\begin{definition}
\label{def:prop}
    A partial \structure $\J$ is a \emph{\T-propagation} of a partial model \I  of \T iff it is an expansion of \I with the same set of model expansions for \T as \I, i.e.,
    \begin{compactitem}
        \item $\I \leq_p \J$
        \item  $\forall \M \in \compl{\I}: (\eval{\T}{\M} = \Tr \implies \M \in \compl{\J})$
    \end{compactitem}
\end{definition}

Structure $\J$ expands $\I$ by selecting some of the uninterpreted symbols whose value is the same in all models of $T$ that expand $\I$, and assigning them that value.
We call $\J$ the  \emph{optimal propagation} of $\I$ with respect to $T$ if it is the most precise \T-propagation.
Algorithms to compute propagation are discussed in the literature (e.g., \cite{DBLP:conf/icfem/ZhangPZ19}).
\ignore{
    \bart{Also: it might be worht mentioning that "the optimal propagation" corresponds to "brave reasoning"}
}

In our interactive assistant framework, a complication arises when propagation derives facts for environmental symbols using the theory of solutions.
In general, there is no guarantee that these facts are true in the environment faced by the user.
To address this issue, we first propagate using \Tenv only, and then using both \Tenv and \Tconf.
This will further demonstrate the benefits of the split theory setting.

Because \Tenv only concerns environmental symbols, the propagation of \Sobs by \Tenv can assign new values only to environmental symbols.
We denote the structure containing these new values by $\Sobs^e$.
The propagation of the state by both \Tenv and \Tconf, i.e., by \Tall, can assign new values to both environmental and decision symbols.
We denote the structures containing these new values by $\Sobs^a$ and $\Sdec^a$.
The structures $\Sobs^e$, $\Sobs^a$ and $\Sdec^a$ are disjoint.
\System displays these structures differently:

\begin{compactitem}
    \item $\Sobs^e$: The new environmental facts obtained by propagation of \Sobs using \Tenv are shown as consequences to \User: they do not need to be verified by observing the environment. This is safe because the propagated facts are true in the environment of \User, by the definition of \Tenv.
    \item $\Sobs^a$: The new environmental facts obtained by propagation of the state using \Tall are shown as prerequisites  of his decisions, to be verified by observing the environment.
        Indeed, there is no guarantee that they are true in the environment faced by \User.
        The observation of the environment may occur at a later stage of the interactive process.
        If the observation does match the value obtained by propagation, \User adds it to the state $\Sobs$; the state will remain consistent.
        If the observation does not match, adding it would make the state inconsistent: \User must first retract a decision to allow the entry, as explained in Section~\ref{process}.
    \item $\Sdec^a$: The new decision facts obtained by propagation of the state using \Tall are shown as consequences to \User. \User cannot just change such a value, since that would cause inconsistency. If \User wants to update the propagated value, they first  need to retract  decisions that led to the propagation.
\end{compactitem}

\begin{example}
    Returning to Ex. 1, the optimal propagation of the state $(\{\LowRent{=}\Fa, $ $\SocialHabitat{=}\Fa\}, $ $\{\})$ is $\Senv^e=\{\}, \Senv^a=\{\}$ and $\Sdec^a= \{\RegistrationType{=}\Other, $ $\TaxRate{=}10\}$: no observation needs to be verified.
    On the other hand, the optimal propagation of $(\{\}, \{\TaxRate=7\})$ is $\Senv^e=\{\}, \Senv^a=\{\LowRent{=}\Tr\}$ and $\Sdec^a= \{\RegistrationType{=}\Modest\}$
    where $\{\LowRent{=}\Tr\}$ needs to be verified.
\end{example}

\section{Partial solutions}

\label{Stop}
We now consider other forms of acceptable solutions.
For example, it may not be necessary to observe the environment fully.
In addition, some decisions may be unnecessary, in the sense that they can be made arbitrarily without impacting the result.
Finally, the user may postpone some decisions when they have the guarantee that a solution exists.

We define definite and contingent solutions.  A \emph{definite} solution entails the satisfaction of \Tconf, assuming \Tenv:

\begin{definition}
    \label{def:DefiniteSolution}
    A \emph{definite solution} of the split model expansion problem $\mathcal{SMX}$ is a state $(\Sobs, \Sdec)$ such that
    (i) all facts in \Sobs are true in \Senv and (ii) \Tconf is satisfied by every expansion \Se of \Sobs that satisfies \Tenv and by every expansion \Sc of \Sdec,
    \[
        \forall\Se \in \compl{\Sobs} : \eval{\Tenv}{\Se} = \Tr
            \implies \forall\Sc \in \compl{\Sdec}: \eval{\Tconf}{\join{\Sc}{\Se}} = \Tr
    \]
\end{definition}

Notice that Definition \ref{def:DefiniteSolution} treats \Tenv and \Tconf very differently.
This is because even though we do not know the full structure \Senv, we do know that it exists, that it is more precise than \Sobs, and that it satisfies \Tenv. For this reason, we only need to consider expansions of \Sobs that satisfy \Tenv.
For such expansions, it should be the case that no matter how we make our choices, we obtain a solution to the problem at hand.

\begin{example}
    A definite solution for the simplified real estate problem without any prior observation is $(\{\}, $ $ \{\RegistrationType{=}\Other, \TaxRate{=}10\})$.
    Indeed, any sale can be registered with type $\Other$ and a tax rate of 10.
\end{example}

Any consistent state can be expanded into a definite solution.
Indeed, by definition, every consistent state can be expanded into a total solution, which is a definite solution.
In practice, \cautious solutions are often partial, not total, structures.
As soon as \User has found a \cautious solution, they can stop making observations and decisions.
Indeed, such a solution can always be expanded to construct total solutions of the problem, by further observing the environmental and making any decision for the remaining decision symbols.

\ignore{
    \begin{proposition}
        \label{stop-criteria}
        If state $(\Sobs, \Sdec)$ is a \cautious solution for the split model expansion problem $\mathcal{SMX}$, then, every total expansion of \Sdec is a total solution for $\mathcal{SMX}$ i.e.,
        \[\forall\Sc \in \compl{\Sdec} :  \eval{\Tconf}{\join{\Senv} {\Sc}} = \Tr \]
    \end{proposition}

    \begin{proof}
        The implication in the definition of a \cautious solution holds for any expansion of \Sobs, and in particular for \Senv.  The result follows from the fact that \Senv satisfies $\Tenv$.
    \end{proof}
}

In some cases, \User only needs the guarantee that a solution is reachable in their environment, by making the appropriate remaining decisions in a later stage.
Since they do not know their environment fully, establishing this guarantee requires considering every environment more precise than what is known:

\begin{definition}
    A \emph{\brave solution} of the split model expansion problem $\mathcal{SMX}$ is a state $(\Sobs, \Sdec)$ such that
    (i) all facts in \Sobs are true in \Senv and (ii) $\Tconf$ is satisfied by at least one expansion \Sc of \Sdec for every expansion \Se of \Sobs that satisfies $\Tenv$,  i.e.,
    \[ \forall\Se\in \compl{\Sobs}: (\eval{\Tenv}{\Se} = \Tr
        \implies \exists\Sc \in \compl{\Sdec} : \eval{\Tconf}{\join{\Sc}{\Se}} = \Tr)
    \]
\end{definition}

\begin{example}
    A contingent solution for the simplified real estate problem without any prior observation is $(\{\}, $ $ \{\RegistrationType = \Other\})$.
    Indeed, any sale can be registered with type $\Other$.  The tax rate can be computed later (in this case easily).
\end{example}

\ignore{
    \begin{proposition}
        \label{stop-criteria2}
        If state $(\Sobs, \Sdec)$ is a \brave solution for the model expansion problem $\mathcal{SMX}$, then, there exists a total solution for $\mathcal{SMX}$ that expands $\Sdec$, i.e.,
        \[\exists \Sc \in \compl{\Sdec} :  \eval{\Tconf}{\join{\Senv} {\Sc}} = \Tr\]
    \end{proposition}%
}

\paragraph{Complexity}  The complexity of deciding whether a state is a \brave or \cautious solution  is much higher than checking whether it is consistent and total.
In the context of propositional logic,  checking  for \cautious solution is $\Pi^p_1$-complete (coNP-complete)  while checking for \brave solution is $\Pi^p_2$-complete. In Section ``Implementation'', we propose a cheap approximate method to solve this problem.

\section{Relevance}
\label{Relevance}

We now turn to the question of determining which observations are relevant for the search.
We have shown that uninterpreted symbols in a \cautious solution $(\Sobs,\Sdec)$ are \emph{irrelevant} (or ``do-not-care'') in the following sense: $\Sdec$ extended with \emph{any} interpretation for these symbols will satisfy $\Tconf$ in \emph{any} possible environment that expands $\Sobs$ while satisfying \Tenv.
In this section, we extend the notion of relevance to any state.

\newpage
\begin{definition}
    A \emph{$\leq_p$-minimal definite solution} is a definite solution that is not more precise than another definite solution.
    We denote the set of minimal definite solutions more precise than a state by $\MinDef(\Sobs,\Sdec)$.
\end{definition}
A minimal contingent solution is similarly defined; the set is denoted by $\MinCont(\Sobs,\Sdec)$.

\begin{definition}
    \label{relevant}
    The symbol $\sigma$ in \Vall is \emph{relevant} in the consistent state $(\Sobs, \Sdec)$ if, and only if, it is interpreted in at least one $\leq_p$-minimal definite solution more precise than the state, i.e.,
    \begin{equation*}
        \exists \I \in \MinDef(\Sobs, \Sdec): \sigma \in \Sigma_{\I}
    \end{equation*}
\end{definition}
By our definition, all symbols interpreted in the state are relevant.
As soon as \User enters a value for a relevant symbol, the state is updated and the list of relevant symbols can be re-computed.

\begin{proposition}
    If there are no uninterpreted relevant symbols in a consistent state, the state is a definite solution.
\end{proposition}

\begin{proof}
     Suppose $(\Sobs, \Sdec)$ is a consistent state without uninterpreted relevant symbols.
     By the definition of consistency, there is a total solution (and hence, a definite) solution that is more precise than $(\Sobs, \Sdec)$, meaning that the set of definite solutions that expand $(\Sobs, \Sdec)$ is non-empty.
     As a consequence (using the fact that $\leqp$ is a well-founded order on the set of states of a fixed and finite vocabulary), this set has at least one $\leqp$-minimal element.
     Let $(\Sobs', \Sdec')$ be such a $\leqp$-minimal definite solution that expands $(\Sobs, \Sdec)$. By Def.~\ref{relevant}, all symbols on which  $(\Sobs', \Sdec')$ and $(\Sobs, \Sdec)$ differ are relevant. Since, by our assumption, all relevant symbols are interpreted in $(\Sobs, \Sdec)$ (and thus also in $(\Sobs', \Sdec')$), the state and the definite solution are actually equal.
\end{proof}

\begin{proposition}
    If there are no uninterpreted relevant environmental symbols in a consistent state, the state is a \brave solution.
\end{proposition}
\begin{proof}
    Similarly to the previous proof, suppose $(\Sobs, \Sdec)$ is a consistent state without uninterpreted relevant environmental symbols.
    Let $(\Sobs', \Sdec')$ be a $\leqp$-minimal contingent solution that expands $(\Sobs, \Sdec)$.
    By the definition, all symbols on which $(\Sobs', \Sdec')$ and $(\Sobs, \Sdec)$ differ are relevant. Since, by our assumption, all relevant environmental symbols are interpreted in $(\Sobs, \Sdec)$ (and thus also in $(\Sobs', \Sdec')$), $\Sobs$ and $\Sobs'$ are equal. Since the state differs from the contingent solution only by the decisions and is consistent, it is a contingent solution too.
\end{proof}

These properties allow \System to alert \User when a definite (resp. contingent) solution has been found.
In propositional logic, the problem of computing whether a symbol $\sigma$ is relevant in some state is $\Sigma^p_3$-complete.
While computing the set of relevant symbols is computationally  expensive, it can be approximated cheaply, as discussed in the next section.

\begin{example}
    When $\{\LowRent{=}\Fa\}$, \LicensedSeller is irrelevant.  Indeed,
    \RegistrationType cannot be \Social nor $\Modest$, and the only $\leq_p$-minimal definite solution is
     $\{\LowRent{=}\Fa, $ $\TaxRate {=} 10, \RegistrationType {=} \Other\}$, in which \LicensedSeller is uninterpreted.
\end{example}

\section {Implementation and  validation}
\label{Implementation}

We developed a decision support tool that embodies our proposal, called the Interactive Consultant.
Given two vocabularies $(\Venv, \Vconf)$ and two theories (\Tenv and \Tconf), the tool generates a user interface that allows the user to follow the proposed interactive process, with the benefit of computer-generated guidance.
The tool supports theories in \fodot, i.e., in first-order logic extended with types, aggregates, linear arithmetic over infinite domains, and (possibly inductive) definitions.\footnote{\url{https://fo-dot.readthedocs.io/}}
It uses the IDP-Z3 reasoning engine for \fodot~\citep{DBLP:journals/corr/abs-2202-00343}\footnote{\url{www.idp-z3.be}}, which is based on the Z3 SMT solver~\citep{de2008z3}.
The tool and its source code are available online.\footnote{\url{https://interactive-consultant.idp-z3.be/} and \url{https://gitlab.com/krr/IDP-Z3}}

The user enters an observation or a decision by clicking on buttons or entering values in fields.
The interactive advisor 1) updates the state, 2) propagates the state using $\Tenv$, 2) propagates again using $\Tenv \land \Tconf$, and 3) determines the relevant symbols.
The user interface is then updated to reflect the attributes and state of symbols: environmental vs. decision symbol, relevant, propagated.
When there are no \emph{environmental} symbols that are both relevant and uninterpreted, the solution is known to be a \brave solution for the environment.
When there is no symbol that is both relevant and uninterpreted, the solution is known to be a \cautious solution for the environment, and the user can stop observing and making decisions.

Because relevance, as defined in Section~\ref{Relevance}, is expensive to compute, we approximate it using a 2-step process:
1) we compute propagation in state $(\Sobs, \Sdec)$ as discussed in Section~\ref{Correctness} and create theory $C$ consisting of the propagated literals;\footnote{I.e., if the $\Tall$-propagation contains $\{p=\Tr, q=\Fa\}$,  theory $C$ contains $\{p. \lnot q.\}$.}
2) we replace atoms occurring in the grounded\footnote{The grounded version is obtained by expanding quantifications and aggregates over their (finite) domain.} version of $\Tall$ by their value in the state (including propagated values), if any, and we simplify the formula using the laws of logic and arithmetic, to obtain $\Tall^*$.
Then, the set of definite solutions that extend state $(\Sobs, \Sdec)$ for theory $C \cup \Tall^*$ is the set of such solutions for theory $\Tall$, i.e., $C \cup \Tall^*$ and $\Tall$ are equivalent in that state.
The symbols that do not occur in theory $C \cup \Tall^*$ cannot be relevant: indeed, changing their value cannot change the truth-value of $\Tall$ in the state.
Thus, theory $C \cup \Tall^*$ contains all the symbols that are relevant in the state; it may also contain some irrelevant symbols, depending on how thoroughly the propagation and simplification is performed, but this is acceptable in practice.
Since Step 1 is computed for the reasons given in Section~\ref{Correctness}, the additional cost of computing relevance is only determined by Step 2, and can be very low and still be useful to the user.


Definitions stated in the theory require special care when determining relevance.
\emph{Definitions} specify a unique interpretation of a defined symbol, given the interpretation of its parameters.
They take the form of an equivalence between the defined atom and a \emph{definiens}, i.e., a logic formula that defines it.
A defined atom can be made irrelevant by replacing it by its definiens in the theory: the transformed theory will be equivalent (except for the defined symbol), and the defined atom will not occur in it anymore.
Yet, defined atoms are often relevant for the user as it abstracts complexity in a useful way.
Thus, we distinguish definitions from the other logical sentences in the theory.
We consider a ground atom relevant if, and only if, it is important for the user to know its value, or it occurs in a simplified logical sentence, or it occurs in the simplified definition of a ground atom previously considered relevant.

\ignore{
    \bart{It would be good to have a formal result of correctness here. But it might not be easy since it is unclear how the "simplification" works}
    \pierre{Exactly.  Maybe the topic of another paper ?}
}

To validate our approach, we have configured the decision support tool for Example~\ref{ex:wet} (available online\footnote{\url{https://tinyurl.com/simpleRegistration}}).
Our tests showed that the system effectively assists the user by correctly determining which hypotheses must be verified, and which observations are relevant.

We then performed experiments on the real estate legislation adapted from~\cite{deryck2019legislation} (available online\footnote{\url{https://tinyurl.com/BelgianRegistration}}).
We simulate the user by a robot entering values in the user interface in a random order.
We measured how many decisions and observations the robot has to make, using two approaches:
\begin{compactitem}
    \item using the traditional approach where the problem instance is fully known before model expansion; the user enters the observations before his decisions, and stops when he finds a total solution;
    \item using our approach; the user only enters relevant observations and decisions, until he finds a definite solution.
\end{compactitem}
In the traditional approach, the user has to make an average of 26.4 entries to obtain a model;\footnote{Depending on the order of data entry, the value of some environmental symbols may be derived by propagation of \Tenv: in that case, the user does not have to enter a value for them.}
with the approach we propose, the user only needs to make 11.6 entries to obtain a definite solution, a 56\% reduction (see details in Appendix).
This reduction in the user's workload is also observed in other realistic use cases~\citep{aerts2022knowledge}.

\section{Related work}
\label{related}

The configuration problem, i.e., the problem of configuring a complex system out of simpler components to satisfy customer requirements, is a special case of model expansion that has been studied, e.g., by~\citet{felfernig2014knowledge}.
Several authors have discussed interactive tools to solve configuration problems~\citep{Hotz2014b, madsen2003methods, van2017kb,falkner2019constraint}.
Our work is orthogonal to these studies.
Indeed, the explicit distinction we make between environmental and decision symbols, and between the theory of the environment and the theory of acceptable solutions, can be combined with the ideas from earlier work.
To illustrate this, our interactive tool supports the 8 reasoning tasks identified as useful by \citet{van2017kb}.

Several approaches for controlling the relevance of symbols in configuration problems have been proposed.
In these approaches, the modeler has the responsibility to write so-called \emph{activation rules} that explicitly specify when a symbol becomes active (as done by \cite{mittal1990dynamic,DBLP:conf/aaai/BowenB91} in Dynamic CSPs), or even when an entire component becomes active (as done by \cite{stumptner1993generative} in Generative CSPs).
Our approach does not need activation rules to derive relevance of symbols.
\citet{jansen2016relevance} also propose an approach without activation rules, based on justification theory.
Its implementation requires a custom solver~\citep{deryck2019legislation}.
By contrast, our implementation uses a standard SMT solver.

To facilitate decision-making in companies, the Object Management Group has published a widely accepted standard in the business process communities, called the Decision Model and Notation standard (DMN)~\citep{DMN}.
One extension of DMN, called cDMN~\citep{aerts2020tackling}, further increases the expressivity of the notation.
Decision models in those two notations can be translated automatically into \fodot~\citep{aerts2020tackling}, allowing our Interactive Consultant to support decision-making in companies.
These DMN representations would be enhanced by using two vocabularies and two theories, as we propose, further highlighting the relevance of our work.

Our two-theory setting has some similarity with distributed constraint satisfaction problems, i.e., problems where several agents, each with their own constraints, search for a solution that satisfies their constraints~\citep{yokoo2012distributed}.
Here, the environment, \System and \User may be viewed as agents, each with their own set of constraints.
In our case, however, the theory of the environment is used by another agent, the interactive assistant.

The concept of relevance has been discussed in a special issue of \textit{Artificial Intelligence} in 1997~\citep{subramanian1997relevance}, focusing on autonomous systems with limited computational resources.
Relevance has been applied to improve the performance of propositional satisfiability solvers~\citep{jansen2016relevance}.
Instead, we apply it to reduce the workload of the user, by showing them the observations and decisions that make them progress in their search of a solution.

\section{Conclusion}

In this paper, we have studied the problem of model expansion in an unknown, but observable, environment in a precise and formal way.
It was crucial to make an \emph{explicit} distinction between environmental and decision symbols, as well as  between the laws of possible environments and the laws of acceptable solutions.
Additionally, we have shown how to reduce, safely, the number of observations and decisions to be made by the user of an interactive assistant in such a setting.

Relevance is an interesting topic for future work.
The partial solutions we describe entail the satisfaction of the whole theory of acceptable solutions, given the environment.
Sometimes, the user is interested in ensuring that a more limited goal is entailed by their decisions.  By determining relevance relative to a goal, the number of observations can be further reduced.  Furthermore, one could investigate whether some observations and decisions are more relevant than others.

Our approach could be generalized to problems with multiple agents participating in a model expansion, each with their own goals and constraints, e.g., in Community-based Configuration problems \citep{Felfernig2014b} and Distributed Constraint Satisfaction~\citep{yokoo2012distributed}.

\section{Competing interests}

The authors declare none.

\bibliographystyle{tlplike}
\bibliography{references}

\appendix

\section{Detailed Validation Report}

    The theory to represent the real estate tax legislation has 27 environmental symbols and two decision symbols.
    We created 5 problem instances by randomly choosing \Venv structures that satisfy \Tenv.

    To simulate the traditional approach, the robot enters the environmental data in a random order, and at each step, propagation by \Tenv is performed.
    If the value of an environmental symbol is derived by this \Tenv-propagation, the robot does not have to enter it.
    After all environmental data is entered, the robot chooses decision values randomly among those that remain possible, until he obtains a total solution.

    To simulate the approach we propose, at each step, the robot fills in one uninterpreted symbol, chosen randomly among the relevant ones, until he obtains a definite solution.  If the symbol is environmental, it enters the value given by the problem instance; if it is a decision symbol, it chooses a valid value at random.  When the state becomes inconsistent, \User retracts a conflicting decision at random.

    Table A 1 shows that the workload is significantly reduced when the system determines which observations are relevant, and which observations must be made in the environment to ensure the solution is definite, as we propose.

    \begin{table}[ht!]
        \label{tab:counts}
     \centering
     \caption{The number of observations and decisions is greatly reduced with our proposal.}
     {\tablefont\begin{tabular}{@{\extracolsep{\fill}}lcc}
       \topline
        ~& \multicolumn{2}{c}{Count of Observations and Decisions} \\
                ~                             & Traditional                          & Our proposal    \\ \hline
                    Sale 1  & 26                              & 10         \\
                    Sale 2  & 26                              & 14                                                                \\
                    Sale 3  & 26                              & 10                                                                \\
                    Sale 4  & 27                              & 10                                                                \\
                    Sale 5  & 27                              & 14                                                                 \\ \hline
                    Average & 26.4                            & 11.6                                                              \\ \hline
                    Gain    & ~           & -14.8   (- 56 \%)
       \botline
        \end{tabular}}
    \end{table}

\end{document}